%% file: main.tex
\PassOptionsToPackage{pagebackref=false, colorlinks=true, citecolor=blue!60!black, linkcolor=black}{hyperref}
\documentclass{article}
\usepackage[utf8]{inputenc}
\usepackage{parskip}
\usepackage[ruled]{algorithm2e}
\usepackage{fullpage}
\usepackage{amsmath}
\usepackage{amsfonts}
\usepackage{amsthm}
\usepackage{mathtools}
\usepackage{bbm}
\usepackage{tikz}
\usepackage{color}
\usepackage{xcolor}
\usepackage{natbib}
\usetikzlibrary{decorations.pathreplacing,calligraphy}

\newcommand\cH{\mathcal{H}}
\newcommand\cX{\mathcal{X}}
\newcommand\cY{\mathcal{Y}}
\newcommand\cD{\mathcal{D}}
\newcommand{\cQ}{\mathcal{Q}}
\newcommand{\cJ}{\mathcal{J}}
\newcommand{\A}{\mathcal{A}}
\newcommand{\cA}{\mathcal{A}}
\newcommand{\E}{\mathop{\mathbb{E}}}
\DeclareMathOperator*{\argmin}{\mathrm{argmin}}

\newtheorem{theorem}{Theorem}[section]

\newtheorem{corollary}[theorem]{Corollary}
\newtheorem{remark}[theorem]{Remark}
\newtheorem{definition}[theorem]{Definition} 
\newtheorem{lemma}[theorem]{Lemma}

\title{Individually Fair Learning with One-Sided Feedback}
\author{
    Yahav Bechavod\thanks{School of Computer Science and Engineering, The Hebrew University. Email: \texttt{yahav.bechavod@cs.huji.ac.il}.}
    \and
    Aaron Roth\thanks{Department of Computer and Information Sciences, University of Pennsylvania. Email: \texttt{aaroth@cis.upenn.edu}.}}%
\date{\today}

\begin{document}

\maketitle

\begin{abstract}
\input{abstract}
\end{abstract}

\newpage
\tableofcontents
\newpage

\input{Introduction}
\input{preliminaries}
\input{reduction}
\input{algorithms}
\input{conclusion}

\section{Acknowledgements}
YB is supported in part by the Apple Scholars in AI/ML PhD Fellowship. AR is supported in part by NSF grant FAI-2147212 and the Simons Collaboration on the Theory of Algorithmic Fairness.

\bibliographystyle{apalike}
\bibliography{refs.bib}

\newpage
\appendix
\input{appendix}

\end{document}

%% file: abstract.tex
We consider an online learning problem with one-sided feedback, in which the learner is able to observe the true label only for positively predicted instances. On each round, $k$ instances arrive and receive classification outcomes according to a randomized policy deployed by the learner, whose goal is to maximize accuracy while deploying \emph{individually fair} policies. We first extend the framework of \cite{BechavodJW20}, which relies on the existence of a human fairness auditor for detecting fairness violations, to instead incorporate feedback from dynamically-selected panels of multiple, possibly inconsistent, auditors. We then construct an efficient reduction from our problem of online learning with one-sided feedback and a panel reporting fairness violations to the contextual combinatorial semi-bandit problem (\cite{Cesa-BianchiL09,GyorgyLLO07}). Finally, we show how to leverage the guarantees of two algorithms in the contextual combinatorial semi-bandit setting: Exp2 \citep{BubeckCK12} and the oracle-efficient Context-Semi-Bandit-FTPL \citep{SyrgkanisKS16}, to provide multi-criteria no regret guarantees simultaneously for accuracy and fairness. Our results eliminate two potential sources
of bias from prior work: the “hidden outcomes” that are not available to an algorithm operating in the full information setting, and human biases that might be present in any single human auditor, but can be mitigated by selecting a well chosen panel.

%% file: Introduction.tex
\section{Introduction}
When making many high stakes decisions about people, we receive only \emph{one-sided} feedback---often we are only able to observe the outcome for people for whom we make a favorable decision. For example, we only observe the repayment history for applicants we approve for a loan---not for those we deny. We only observe the success or lack thereof for employees we hire, not for those that we pass on. We only observe the college GPA for those applicants that we admit to college, not to those we reject---and so on. In all of these domains, fairness is a major concern in addition to accuracy. Nevertheless, the majority of the literature on fairness in machine learning does not account for this ``one-sided'' feedback structure, operating either in a batch setting, a full information online setting, or in a more standard bandit learning setting. But when we make sequential decisions with one-sided feedback, it is crucial to explicitely account for the form of the feedback structure to avoid feedback loops that may amplify and disguise historical bias.

The bulk of the literature in algorithmic fairness also asks for fairness on a \emph{group} or aggregate level. A standard template for this approach is to select some statistical measure of error (like false positive rate, false negative rates, or raw error rates), a partition of the data into groups (often along the lines of some ``protected attribute''), and then to ask that the statistical measure of error is approximately equalized across the groups. Because these guarantees bind only over averages over many people, they promise little to individuals, as initially pointed out by Dwork et al.'s ``catalogue of evils'' \citep{DworkHPRZ12}.

In an attempt to provide meaningful guarantees on an individual level, \citet{DworkHPRZ12} introduced the notion of individual fairness, which informally asks that ``similar individuals be treated similarly''. In their conception, this is a Lipschitz constraint imposed on a randomized classifier, and who is ``similar'' is defined by a task-specific similarity metric. Pinning down such a metric is the major challenge with using the framework of individual fairness. \cite{GillenJKR18} proposed that feedback could be elicited in an online learning setting from a human auditor who ``knows unfairness when she sees it'' (and implicitly makes judgements according to a similarity metric), but cannot enunciate a metric --- she can only identify specific violations of the fairness constraint. Recently, \citet{BechavodJW20} gave an algorithm for operating in this setting---with full information--- that was competitive with the optimal fair model, while being able to learn not to violate the notion of individual fairness underlying the feedback of a single auditor. 

Our work extends that of \citet{GillenJKR18,BechavodJW20} in two key ways. First, we remove the assumption of a single, consistent auditor: we assume we are given an adaptively chosen \emph{panel} of human auditors who may have different conceptions of individual fairness and may be making inconsistent judgements (we aim to be consistent with plurality judgements of such a panel). Second, we dispense with the need to operate in a full information setting, and give oracle efficient algorithms that require only one-sided feedback. We give simultaneous no-regret guarantees for both classification error and fairness violation, with respect to models that are individually fair in hindsight (i.e. given the realization of the panels of fairness auditors who define our conception of fairness). Together these improvements eliminate two potential sources of bias from prior work: the ``hidden outcomes'' that are not available to an algorithm operating in the full information setting, and human biases that might be present in any single human auditor, but can be mitigated by selecting a well chosen panel.

\subsection{Overview of Results}
We define an online learning framework with one-sided label feedback and additional feedback from dynamically chosen panels of auditors regarding fairness violations (we present our formal model in Section  \ref{sec:prelim}). We then cast our learning problem as an optimization problem of a joint objective using a Lagrangian formulation (Section \ref{subsec:lagrangian}). We construct an efficient reduction to the contextual combinatorial semi-bandit setting \citep{Cesa-BianchiL09,GyorgyLLO07}, allowing us to upper bound the Lagrangian regret in our setting (Section  \ref{sec:reduction}). We show that auditing by panels is in fact equivalent to auditing by specific, instance-dependent, ``representative'' auditors, which is a useful technical step in our analysis (Section \ref{subsec:equivalence}). We then show how the Lagrangian regret guarantee can be translated to achieve multi-criteria no regret guarantees, simultaneously for accuracy and fairness (Section \ref{subsec:simultaneous}).
Finally, we show how to leverage the regret guarantees of two algorithms for the contextual combinatorial semi-bandit setting: Exp2 \citep{BubeckCK12} and the oracle-efficient Context-Semi-Bandit-FTPL \citep{SyrgkanisKS16}, to establish no regret guarantees simultaneously for each of accuracy and fairness (Section  \ref{subsec:algorithms}).

\subsection{Related Work}
Our work is  related to two strands of literature: learning with one-sided feedback, and individual fairness in machine learning.  The problem of learning from positive-prediction-only feedback first appeared in \citet{Helmbold}, under the name of ``apple tasting''. Subsequently, \citet{PartialMonitoring} studied a generalization of the one-sided feedback setting, in which the feedback at each round is a function of the combined choice of two players. Follow-up work by \citet{AntosBPS13} showed that it is possible to reduce the online one-sided feedback setting to the better studied contextual bandit problem. \citet{cesa-bianchi06b} focuses on linear models, and propose an active learning approach based on the predictions made by the deployed predictor at each round,  in the face of one-sided feedback. \citet{Sculley07} focused on practical challenges in learning with one-sided feedback in the context of spam filtering, and suggested the utilization of methods form the active learning literature to reduce the label complexity encountered in practice. \citet{JiangJP21} focuses on learning with generalized linear models in an online one-sided feedback setting, and propose a data-driven adaptive approach based on variance estimation techniques. \citet{SelectiveExpert} and \citet{Lakkaraju17} propose techniques for imputing missing labels using feedback from human experts. \citet{Ustun17} and \citet{LakkarajuRudin17} propose statistical techniques for assigning missing labels.

In the context of algorithmic fairness, \citet{Bechavod19} considers a stochastic online setting with one-sided feedback, in which the aim is to learn a binary classifier while enforcing the statistical fairness condition of ``equal opportunity'' \citep{hardt16}. \citet{CostonRC21} operate in a batch setting with potentially missing labels due to one-sided feedback in historical decisions, and attempt to impute missing labels using statistical techniques. \citet{ensign18a} and \cite{elzayn2019fair} focus on the tasks of predictive policing and related resource allocation problems, and give algorithms for these tasks under a censored feedback model. \citet{Kleinberg17} explores techniques to mitigate the problem of one-sided feedback in the context of judicial bail decisions.

\citet{DworkHPRZ12} introduced the notion of individual fairness. In their formulation, a similarity metric is explicitly given, and they ask that predictors satisfy a Lipschitz condition (with respect to this metric) that roughly translates into the condition that ``similar individuals should have similar distributions over outcomes''.
\citet{YonaR18} give a statistical treatment of individual fairness in a batch setting with examples drawn i.i.d. from some distribution, while assuming full access to a similarity metric, and prove PAC-style generalization bounds for both accuracy and individual fairness violations. \citet{Ilvento} suggests learning the similarity metric from human arbiters, using a hybrid model of comparison queries and numerical distance queries. \citet{KimRR18} study a group-based relaxation of individual fairness, while relying on access to an auditor returning unbiased estimates of distances between pairs of individuals. \citet{JungKNRSW21} consider a batch setting, with a fixed set of ``stakeholders'' which provide fairness feedback regarding pairs of individuals in a somewhat different model of fairness, and give oracle-efficient algorithms and generalization bounds. \citet{GuptaK19} study a time-dependent variant of individual fairness they term ``individual fairness in hindsight''. \citet{YurochkinBS20} consider a variant of individual fairness which asks for invariance of the learned predictors with respect to ``sensitive'' variables. \citet{MukherjeeYBS20} investigate ways to learn the metric from data. \citet{LahotiGW19} focus on the task of learning individually fair representations. \citet{JosephKMR16,JosephKMNR18} study a notion of individual fairness in a more traditional contextual bandit setting, in which $k$ individuals arrive at each round, and some subset of them are ``selected'', which yields observable reward to the learner. Their notion of individual fairness mandates that selection probability should be monotone in the (true) label of an individual (and in particular individuals with the same true label should be selected with the same probability). True labels cannot in general be ascertained, and as a result they only give positive results under strong realizability assumptions.

The papers most related to ours are \citet{GillenJKR18} and \citet{BechavodJW20}. \citet{GillenJKR18} introduces the idea of online learning with human auditor feedback as an approach to individual fairness, but give algorithms that are limited to a single auditor that makes decisions with respect to a restrictive parametric form of fairness metrics in the full information setting. \citet{BechavodJW20} generalize this to a much more permisive definition of a human auditor, but still operate in the full information setting and are limited to single human auditors.

%% file: preliminaries.tex
\section{Preliminaries}\label{sec:prelim}

We start by specifying the notation we will use for our setting. We denote a feature space by $\cX$ and a label space by $\cY$. Throughout this work, we focus on the case where $\cY = \{0,1\}$. We denote by $\cH$ a hypothesis class of binary predictors $h:\cX\rightarrow\cY$, and assume that $\cH$ contains a constant hypothesis. For the purpose of achieving better accuracy-fairness trade-offs, we allow the deployment of randomized policies over the base class $\cH$, which we denote by $\Delta \cH$. As we will see later, in the context of individual fairness, it will be crucial to be able to compete with the best predictor in $\Delta\cH$, rather than simply in $\cH$. We model auditors as  observing $k$-tuples of examples (the people who are present at some round of the decision making process), as well as our randomized prediction rule, and will form objections by identifying a pair of examples for which they believe our treatment was ``unfair'' if any such pair exists. For an integer $k\geq 2$, we denote by $\cJ:\Delta\cH\times\cX^k\rightarrow\cX^2$ the domain of possible auditors. Next, we formalize the notion of fairness we will aim to satisfy.

\subsection{Individual Fairness and Auditing by Panels}

Here we define the notion of individual fairness and auditing that we use, following  \citet{DworkHPRZ12,GillenJKR18,BechavodJW20}, and extending it to the notion of a panel of auditors.

\begin{definition}[$\alpha$-fairness violation]
Let $\alpha\geq0$ and let $d:\cX\times\cX\rightarrow[0,1]$.\footnote{
 $d$ represents a function specifying the auditor's judgement of the ``similarity'' between individuals in a specific context. We do not require that $d$ be a metric: only that it be non-negative and symmetric. It is important that we make as few assumptions as possible when modeling human auditors, as in general, we cannot expect this form of feedback to take specific parametric form, or even be a metric.} We say that a policy $\pi\in\Delta\cH$ has an $\alpha$-fairness violation (or simply ``$\alpha$-violation'') on $(x,x')\in\cX^2$  with respect to $d$ if
\[
\pi(x)-\pi(x') > d(x,x') + \alpha.
\]
where $\pi(x) = \Pr_{h \sim \pi}[h(x) = 1]$.
\end{definition}

A fairness auditor, parameterized by a distance function $d$, given a policy $\pi$ and a set of $k$ individuals, will report any single pair of the $k$ individuals on which $\pi$ represents an $\alpha$-violation if one exists.

\begin{definition}[Auditor]
Let $\alpha\geq 0$. We define a fairness auditor $j^\alpha\in\cJ$ by, $\forall \pi\in\Delta\cH, \bar{x}\in\cX^k$,
\[
j^\alpha\left(\pi,\bar{x}\right) := \begin{cases}
(\bar{x}^{s},\bar{x}^l)\in V^{j} &\text{if } V^j := \{(\bar{x}^s,\bar{x}^l):s\neq l\in[k],\\
&\pi(\bar{x}^s)-\pi(\bar{x}^l) > d^j(x,x') + \alpha\}\neq\emptyset\\
(v,v) &\text{otherwise}
\end{cases},
\]

where $\bar{x} = (\bar{x}^1,\dots,\bar{x}^k)$, $d^j:\cX\times\cX\rightarrow[0,1]$ is auditor $j^\alpha$'s (implicit) distance function, and $v\in\cX$ is some ``default'' context. When clear from context, we will abuse notation, and simply use $j$ to denote such an auditor.
\end{definition}

Note that if there exist multiple pairs in $\bar{x}$ on which an $\alpha$-violation exist, we only require the auditor to report one. In the case in which the auditor does not consider there to be any fairness violations, we define its output to be a ``default'' value, $(v,v)\in\cX^2$, to indicate that no violation was detected.

Thus far our formulation of fairness violations and auditors follows the formulation in \citet{BechavodJW20}. In the following, we  generalize the notion of fairness violations to panels of multiple fairness auditors which extends beyond the framework of \citet{BechavodJW20}.

\begin{definition}[$(\alpha,\gamma)$-fairness violation] \label{def:panel-violation}
Let $\alpha\geq 0$, $0\leq\gamma\leq 1$, $m\in \mathbb{N}\setminus\{0\}$. We say that a policy $\pi\in\Delta\cH$ has an $(\alpha,\gamma)$-fairness violation (or simply ``$(\alpha,\gamma)$-violation'') on $(x,x')\in\cX^2$ with respect to $d^1,\dots,d^m:\cX^2\rightarrow[0,1]$ if
\[
\frac{1}{m}\sum_{i=1}^m \mathbbm{1}\left[\pi(x)-\pi(x') - d^i(x,x') > \alpha\right]\geq\gamma.
\]
\end{definition}

Definition \ref{def:panel-violation} specifies that a policy $\pi$ has an $(\alpha,\gamma)$-fairness violation on a pair of examples when a $\gamma$ fraction of the auditors consider $\pi$ to have an $\alpha$-fairness violation on that pair. By varying $\gamma$, we can interpolate between considering there to be a violation when any \emph{single} auditor determines that there is one at one extreme, to requiring unanimity amongst the auditors at the other extreme.

\begin{definition}[Panel]\label{def:panel}
Let $\alpha\geq 0$, $0\leq\gamma\leq 1$, $m\in \mathbb{N}\setminus\{0\}$. We define a fairness panel $\bar{j}^{\alpha,\gamma}$ by, $\forall \pi\in\Delta\cH, \bar{x}\in\cX^k$,
\[
\bar{j}^{\alpha,\gamma}_{j^1,\dots,j^m}(\pi,\bar{x}) = \begin{cases}
(\bar{x}^{s},\bar{x}^l)\in V^{\bar{j}} &\text{if } V^{\bar{j}} := \{(\bar{x}^{s},\bar{x}^l):s\neq l\in[k]\land \exists i_1,\dots,i_{\lceil\gamma m\rceil}\in[m],\\
&\forall s\in[\lceil\gamma m\rceil],
(\bar{x}^{s},\bar{x}^l)\in V^{j^{i_s}}\}\neq\emptyset\\
(v,v) &\text{otherwise}
\end{cases},
\]
where $\bar{x} := (\bar{x}^1,\dots,\bar{x}^k)$, $d^j:\cX\times\cX\rightarrow[0,1]$ is auditor $j$'s (implicit) distance function, and $v\in\cX$ is some ``default'' context. When clear from context, we will abuse notation and simply denote such a panel by $\bar{j}$.
\end{definition}

Again panels need only report a \emph{single} $(\alpha,\gamma)$-violation even if many exist. The rationale behind extending the auditing scheme to panels is that human auditors have their own implicit biases, and so there may be no single human auditor that a collection of stakeholders would agree to entrust with fairness judgements. It is much easier to agree on a representative panel of authorities. As already noted, the $\gamma$ parameter allows us to adjust the degree to which we require consensus amongst panel members: we can interpolate all the way between requiring full unanimity on all judgements of unfairness (when $\gamma = 1$) to giving any single panel member effective ``veto power'' (when $\gamma \leq 1/m$).

\begin{remark}[On Exploring the Accuracy-Fairness Tradeoff Frontier]
Note that different values of $\gamma$ for the panel do not change the auditing task for individual auditors: in all cases, each auditor is only asked to report $\alpha$-violations according to their own judgement. Thus, using the same feedback from a panel of auditors, we can algorithmically vary $\gamma$ to explore an entire frontier of fairness/accuracy tradeoffs. 
\end{remark}

\subsection{Lagrangian Loss  Formulation} \label{subsec:lagrangian}
Next, we define the three types of loss we will use in our setting.

\begin{definition}[Misclassification loss] We define the misclassification loss as, for all $\pi\in\Delta\cH$,  $\bar{x} \in \cX^k$, $\bar{y}\in\{0,1\}^k$ as:
\[
Error(\pi,\bar{x},\bar{y}) := \E\limits_{h\sim\pi} [\ell^{0-1}(h,\bar{x},\bar{y})].
\]
Where for all $h\in\cH$, $\ell^{0-1}(h,\bar{x},\bar{y}) := \sum_{i=1}^k \ell^{0-1}(h,(\bar{x}^i,\bar{y}^i))$, and $\forall i\in[k]:\ell^{0-1}(h,(\bar{x}^i,\bar{y}^i)) = \mathbbm{1}[h(\bar{x}^i)\neq \bar{y}^i]$.
\end{definition}

Next, we define the unfairness loss, to reflect the existence of one or more fairness violations according to a panel's judgement.

\begin{definition}[Unfairness loss] Let $\alpha\geq 0$, $0\leq\gamma\leq 1$.
We define the unfairness loss as, for all $\pi\in\Delta\cH$, $\bar{x} \in \cX^k$, $\bar{j}=\bar{j}^{\alpha,\gamma}_{j^1,\dots,j^m}:{\cX^k}\rightarrow\cX^2$,
\[
Unfair^{\alpha,\gamma}(\pi,\bar{x},\bar{j}) := \begin{cases}
1 &\bar{j}(\pi,\bar{x}) = (\bar{x}^s,\bar{x}^l) \land s\neq l
\\
0 &\text{otherwise}
\end{cases},
\]
where $\bar{x} := (\bar{x}^1,\dots,\bar{x}^k)$.
\end{definition}

Finally, we define the Lagrangian loss.

\begin{definition}[Lagrangian loss]\label{def:Lagrangian}
Let $C>0$, $\rho = (\rho^1,\rho^2)\in\cX^2$. We define the $(C,\rho)$-Lagrangian loss as, for all $\pi\in\Delta\cH$,  $\bar{x} \in \cX^k$, $\bar{y}\in\{0,1\}^k$,
\[
L_{C,\rho}(\pi,\bar{x},\bar{y}) := Error(\pi,\bar{x},\bar{y}) + C\cdot\left[\pi(\rho^1)-\pi(\rho^2)\right].
\]
\end{definition}

We are now ready to formally define our learning environment, which we do next.

\subsection{Individually Fair Online Learning with One-Sided Feedback}

In this section, we formally define our learning environment with one-sided feedback. The interaction proceeds in a sequential fashion, where in each round, the learner first deploys a policy $\pi^t\in\Delta\cH$, then the environment selects $k$ individuals $\bar{x}^t \in\cX^k$, and their labels $\bar{y}^t\in\cY^k$, possibly in an adversarial fashion. The learner is only shown $\bar{x}^t$. The environment then selects a panel of auditors $(j^{t,1},\dots,j^{t,m})\in\cJ^m$, possibly in adversarial fashion. The learner predicts $\bar{x}^t$ according to  $\pi^t$. The panel then reports whether a fairness violation was found, according to the predictions made by $\pi^t$. Finally, the learner  observes the true label only for positively-predicted individuals in $\bar{x}^t$, and suffers two types of loss: a misclassification loss (note that this loss may \emph{not} be observable to the learner due to the one-sided feedback) and an unfairness loss. Our setting is summarized in Algorithm \ref{alg:apple-tasting}.

\paragraph{One-sided feedback}
Our one-sided feedback structure (classically known as ``apple tasting'') is different from the standard bandit setting. In the bandit setting, the feedback visible to the learner is the loss for the selected action in each round. In our setting,  feedback may or may not be observable for a selected action: if we classify an individual as positive, we observe feedback for our action---and for the counterfactual action we could have taken (classifying them as negative). On the other hand, if we classify an individual as negative, we do not observe (but still suffer) our classification error.

\begin{algorithm}[ht]
\caption{Individually Fair Online Learning with One-Sided Feedback}\label{alg:apple-tasting}
\SetAlgoLined
\textbf{Input:} Number of rounds $T$, hypothesis class $\cH$\;
Learner initializes $\pi^1 \in \Delta \cH$\;
\For{$t = 1, \ldots, T$}{
Environment selects individuals $\bar{x}^t \in\cX^k$, and labels $\bar{y}^t\in\cY^k$, learner only observes $\bar{x}^t$\;
Environment selects panel of auditors $(j^{t,1},\dots,j^{t,m})\in\cJ^m$ \;
Learner draws $h^t\sim\pi^t$, predicts $\hat{y}^{t,i} = h^t(\bar{x}^{t,i})$ for each $i \in [k]$, observes $\bar{y}^{t,i}$ iff $\hat{y}^{t,i} = 1$\;
Panel reports its feedback $\rho^t = \bar{j}^{t,\alpha,\gamma}_{j^1,\dots,j^m}(\pi^t,\bar{x}^t)$ \;
Learner suffers misclassification loss $Error(h^t,\bar{x}^t, \bar{y}^t)$ (not necessarily observed by learner)\;
Learner suffers unfairness loss $Unfair(\pi^t,\bar{x}^t,\bar{j}^t)$\;
Learner updates $\pi^{t+1} \in \Delta \cH$\;
}
\end{algorithm}

To measure performance, we will ask for algorithms who are competitive with the best possible (fixed) policy in hindsight. This is captured using the notion of regret, which we define next for relevant loss functions.

\begin{definition}[Error regret]\label{def:Error-Regret}

We define the error regret of an algorithm $\cA$ against a comparator class $U\subseteq\Delta\cH$ to be
\[
Regret^{err}(\cA,T,U) = \sum_{t=1}^T Error(\pi^t,\bar{x}^t,\bar{y}^t) - \min\limits_{\pi^*\in U} \sum_{t=1}^T Error(\pi^*,\bar{x}^t,\bar{y}^t).
\]
\end{definition}

\begin{definition}[Unfairness regret]\label{def:Unfairness-Regret}
Let $\alpha\geq 0$, $0\leq\gamma\leq 1$. We define the unfairness regret of an algorithm $\cA$ against a comparator class $U\subseteq\Delta\cH$ to be
\[
Regret^{unfair,\alpha,\gamma}(\cA,T,U) = \sum_{t=1}^T Unfair^{\alpha,\gamma}(\pi^t,\bar{x}^t,\bar{j}^t)
- \min\limits_{\pi^*\in U} \sum_{t=1}^T Unfair^{\alpha,\gamma}(\pi^*,\bar{x}^t,\bar{j}^t).
\]
\end{definition}

Finally, we define the Lagrangian regret, which will be useful in our analysis.

\begin{definition}[Lagrangian regret]\label{def:Lagrangian-Regret}
Let $C>0$, and $(\rho^t)_{t=1}^T$ be a sequence s.t. $\forall t\in[T]:\rho^t\in\cX^2$. We define the Lagrangian regret of an algorithm $\cA$ against a comparator class $U\subseteq\Delta\cH$ to be
\[
Regret^{L,C,\rho^1,\dots,\rho^T}(\cA,T,U) =  \sum_{t=1}^T L_{C,\rho^t}(\pi^t,\bar{x}^t,\bar{y}^t) - \min_{\pi^*\in U} \sum_{t=1}^T L_{C,\rho^t}(\pi^*,\bar{x}^t,\bar{y}^t).
\]
\end{definition}

In order to construct an algorithm that achieves no regret simultaneously for accuracy and fairness, our approach will be to reduce the setting of individually fair learning with one-sided feedback (Algorithm \ref{alg:apple-tasting}) to the setting of contextual combinatorial semi-bandit, which we will see next.

%% file: reduction.tex
\section{Reduction to Contextual Combinatorial Semi-Bandit}\label{sec:reduction}

In this section, we present our main result: a reduction from individually fair online learning with one-sided feedback (Algorithm \ref{alg:apple-tasting}) to the setting of (adversarial) contextual combinatorial semi-bandit.

\subsection{Contextual Combinatorial Semi-Bandit}

We begin by formally describing the contextual combinatorial semi-bandit setting.\footnote{The combinatorial (full) bandit problem formulation is due to \citet{Cesa-BianchiL09}. We consider a contextual variant of the problem. Our setting operates within a relaxation of the feedback structure, known as ``semi-bandit'' \citep{GyorgyLLO07}.} The setting can be viewed as an extension of the classical $k$-armed contextual bandit problem, to a case where $k$ instances arrive every round, each to be labelled as either $0$ or $1$. However, at each round, the action set over these labellings is restricted only to a subset  $A^t\subseteq\{0,1\}^k$, where each action corresponds the vector containing the predictions of a hypothesis $h\in\cH$ on the arriving contexts. Finally, the learner suffers loss that is a linear function of all specific losses for each of the $k$ instances, and is restricted to only observe the coordinates of the loss vector on instances predicted as $1$. The setting is summarized in Algorithm  \ref{alg:ccsb}.

\begin{algorithm}[ht]
\caption{Contextual Combinatorial Semi-Bandit}\label{alg:ccsb}
\SetAlgoLined
\textbf{Parameters:} Class of predictors $\cH$, number of rounds $T$\;
Learner deploys $\pi^1\in\Delta\cH$\;
\For{$t = 1, \ldots, T$}{
Environment selects loss vector $\ell^t \in [0,1]^k$ (without revealing it to learner)\;
Environment selects contexts $\bar{x}^t\in\cX^k$, and reveals them to the learner\;
Learner draws action $a^t \in A^t \subseteq \{0,1\}^k$ according to $\pi^t$ (where $A^t=\{a^{t}_h = (h(\bar{x}^{t,1}),\dots,h(\bar{x}^{t,k})):\forall h\in\cH\}$) \;
Learner suffers linear loss $\left\langle a^t,\ell^t\right\rangle$\;
Learner observes $\ell^{t,i}$ iff $a^{t,1} = 1$\;
Learner deploys $\pi^{t+1}$\;
}
\end{algorithm}

We next define regret in the context of contextual combinatorial semi-bandit (Algorithm \ref{alg:ccsb}).

\begin{definition}[Regret] \label{def:regret}
In the setting of Algorithm \ref{alg:ccsb}, we define the regret of an algorithm $\cA$ against a comparator class $U\subseteq\Delta\cH$
to be
\[
Regret(\cA,T,U) = \sum_{t=1}^T \E\limits_{a^t\sim\pi^t}\left\langle a^t,\ell^t \right\rangle - \min\limits_{\pi^*\in U} \sum_{t=1}^T \E\limits_{a^*\sim\pi^*}\left\langle a^t,\ell^t \right\rangle.
\]
\end{definition}

\subsection{Reduction}
Our reduction consists of two major components: encoding the fairness constraints, and translating one-sided feedback to semi-bandit feedback. We start by extending the sample set at round $t$, to encode $C$ copies of each of the individuals in the pair reported by the panel, where we append label of $0$ to the first individual, and label of $1$ to the second, in order to ``translate'' unfairness into error. We next translate one-sided feedback to semi-bandit feedback, by constructing the first half of the loss vector $\ell^t$, to return a loss of $0$ or $1$ for positively-predicted instances (according to the, observable true label), and the second half to return a loss of $1/2$ for negatively-predicted instances (regardless of the true label, which is unobservable, since the prediction is negative). We note that this transformation of the standard $0-1$ loss is regret-preserving, and the resulting losses are also linear and non-negative (which will be important in our setting, as we will see in Section \ref{sec:algorithms}). Finally, the first half of the action vector is constructed by simply invoking the selected hypothesis by the learner, $h^t$, on each of the contexts in the augmented $\bar{\bar{x}}^t$, while the second half reflects the opposites of the predictions made in the first half. The reduction is summarized in Algorithm \ref{alg:reduction}.

In describing the reduction, we use the following notations (For integers $k\geq 2$, $C\geq 1$):
\begin{align*}
&(i)~\forall a\in \{\rho^{t,1},\rho^{t,2},0,1,1/2\}:\quad
\bar{a} := \overbrace{(a,\dots,a)}^{\text{C times}},\quad \bar{\bar{a}} :=\overbrace{(a,\dots,a)}^{\text{k+2C times}}.\\
&(ii)~h(\bar{\bar{x}}^t) := (h(\bar{\bar{x}}^{t,1}),\dots,h(\bar{\bar{x}}^{t,2k+4C})).
\end{align*}

\begin{algorithm}[H]
\caption{Reduction to Contextual Combinatorial Semi-Bandit}\label{alg:reduction}
\SetAlgoLined
\textbf{Input:} Contexts $\bar{x}^t\in\cX^k$, labels $\bar{y}^t\in\{0,1\}^k$, hypothesis $h^t$, pair $\rho^t\in\cX^2$, parameter $C\in\mathbb{N}$\;
\begin{tabular}{l l}
\renewcommand{\arraystretch}{2}
\\[-0.2cm]
$\text{Define}$: & $\bar{\bar{x}}^t= (\bar{x}^t,\bar{{\rho}}^{t,1},\bar{{\rho}}^{t,2})\in\cX^{k+2C}, \quad \bar{\bar{y}}^t = (\bar{y}^t,\bar{0},\bar{1})\in\{0,1\}^{k+2C}$;\\[0.2cm]
$\text{Construct loss vector}$: & $\ell^t= (\bar{\bar{1}} -\bar{\bar{y}}^t,\bar{\bar{1/2}})\in[0,1]^{2k+4C}$;\\[0.2cm]
$\text{Construct action vector}$: & $a^t = (h^t(\bar{\bar{x}}^t),\bar{\bar{1}}-h^t(\bar{\bar{x}}^t))\in\{0,1\}^{2k+4C};$\\[0.2cm]
\end{tabular}\\
\textbf{Output:} $(\ell^t, a^t)$\;
\end{algorithm}

We next prove that the reduction described in Algorithm \ref{alg:reduction} can be used to upper bound an algorithm's  Lagrangian regret in the individually fair online learning with one-sided feedback setting, within a multiplicative factor of 2 times the dimension of the output of the reduction.

For the following theorem, we will assume the existence of an algorithm $\cA$ for the contextual combinatorial semi-bandit setting (as summarized in Algorithm \ref{alg:ccsb}) whose expected regret (compared to only fixed hypotheses in  $\cH$), against any adaptively and adversarially chosen sequence of loss functions $\ell^t$ and contexts $\bar{x}^t$, is bounded by $Regret(\cA,T,\cH)\leq R^{\cA,T,\cH}$.

\begin{theorem}\label{thm:reduction} In the setting of individually fair online learning with one-sided feedback (Algorithm \ref{alg:apple-tasting}), running $\cA$ while using the sequence $(a^t,\ell^t)_{t=1}^T$ generated by the reduction in Algorithm \ref{alg:reduction} (when invoked every round on $\bar{x}^t$, $\bar{y}^t$, $h^t$, $\rho^t$, and $C$), yields the following guarantee, for any $V \subseteq \Delta\cH$,
\[
\sum_{t=1}^T L_{C,\rho^t}(\pi^t,\bar{x}^t,\bar{y}^t) - \min_{\pi^*\in V} \sum_{t=1}^T L_{C,\rho^t}(\pi^*,\bar{x}^t,\bar{y}^t)\leq 2(2k+4C)R^{\cA,T,\cH}.
\]
\end{theorem}

In order to prove Theorem \ref{thm:reduction}, we first state and prove two lemmas, which express the Lagrangian regret in the setting of Individually Fair Online Learning with One-Sided Feedback (Algorithm \ref{alg:apple-tasting}) in terms of the regret in the contextual combinatorial semi-bandit setting (Algorithm \ref{alg:ccsb}).
In what follows, we denote $k' = k+2C$.

\begin{lemma} \label{lem:reduction1}
For all $\pi,\pi'\in\Delta\cH$, $\bar{x}^t\in\cX^k$, $\bar{y}^t\in\{0,1\}^k$, 
\[
L_{C,\rho^t}(\pi,\bar{x}^t,\bar{y}^t) - L_{C,\rho^t}(\pi',\bar{x}^t,\bar{y}^t) = 
\sum_{i=1}^{k'} Error(\pi,\bar{\bar{x}}^{t,i},\bar{\bar{y}}^{t,i}) - \sum_{i=1}^{k'} Error(\pi',\bar{\bar{x}}^{t,i},\bar{\bar{y}}^{t,i}).
\]
\end{lemma}

\begin{proof}
Observe that
\begin{align*}
&L_{C,\rho^t}(\pi,\bar{x}^t,\bar{y}^t) - L_{C,\rho^t}(\pi',\bar{x}^t,\bar{y}^t)\\
&= Error(\pi,\bar{x}^t,\bar{y}^t) + C\cdot\left[\pi(\rho^{t,1})-\pi(\rho^{t,2})\right] - Error(\pi',\bar{x}^t,\bar{y}^t) - C\cdot\left[\pi'(\rho^{t,1})-\pi'(\rho^{t,2})\right]\\
&= \sum_{i=1}^{k} Error(\pi,\bar{\bar{x}}^{t,i},\bar{\bar{y}}^{t,i}) - Error(\pi',\bar{\bar{x}}^{t,i},\bar{\bar{y}}^{t,i}) + \sum_{i=k+1}^{k+C} \pi(\rho^{t,1})-\pi'(\rho^{t,1}) + \sum_{i=k+C+1}^{k+2C} 1-\pi(\rho^{t,2})-1+\pi'(\rho^{t,2})\\
&= \sum_{i=1}^{k} Error(\pi,\bar{\bar{x}}^{t,i},\bar{\bar{y}}^{t,i}) -  Error(\pi',\bar{\bar{x}}^{t,i},\bar{\bar{y}}^{t,i}) + \sum_{i=k+1}^{k+C} Error(\pi,\bar{\bar{x}}^{t,i},\bar{\bar{y}}^{t,i})-Error(\pi',\bar{\bar{x}}^{t,i},\bar{\bar{y}}^{t,i})\\
&+ \sum_{i=k+C+1}^{k+2C} Error(\pi,\bar{\bar{x}}^{t,i},\bar{\bar{y}}^{t,i})-Error(\pi',\bar{\bar{x}}^{t,i},\bar{\bar{y}}^{t,i})\\
&=\sum_{i=1}^{k'} Error(\pi,\bar{\bar{x}}^t,\bar{\bar{y}}^t) - \sum_{i=1}^{k'} Error(\pi',\bar{\bar{x}}^t,\bar{\bar{y}}^t).
\end{align*}
Which proves the lemma.
\end{proof}

\begin{lemma} \label{lem:reduction2}
For all $\pi,\pi'\in\Delta\cH$, $\bar{\bar{x}}^t\in\cX^{k'}$,$\bar{\bar{y}}^t\in\cY^{k'}$, 
\[
\sum_{i=1}^{k'} Error(\pi,\bar{\bar{x}}^t,\bar{\bar{y}}^t) - \sum_{i=1}^{k'} Error(\pi',\bar{\bar{x}}^t,\bar{\bar{y}}^t) = 2\left[\E\limits_{h\sim\pi}\left[\langle a^{h},\ell^t\rangle\right] - \E\limits_{h'\sim\pi'}\left[\langle a^{h'},\ell^t\rangle\right]\right].
\]
\end{lemma}

\begin{proof}
Observe that
\begin{align*}
&\sum_{i=1}^{k'} Error(\pi,\bar{\bar{x}}^t,\bar{\bar{y}}^t) - \sum_{i=1}^{k'} Error(\pi',\bar{\bar{x}}^t,\bar{\bar{y}}^t)\\ 
&= \left[\sum_{i=1}^{k'} Error(\pi,\bar{\bar{x}}^{t,i},\bar{\bar{y}}^{t,i}) +\mathbbm{1}\left[\bar{\bar{y}}^{t,i} = 0\right]\right] - \left[\sum_{i=1}^{k'} Error(\pi',\bar{\bar{x}}^{t,i},\bar{\bar{y}}^{t,i}) + \mathbbm{1}\left[\bar{\bar{y}}^{t,i} = 0\right]\right]\\
&=2\bigg[\bigg\langle\left(\pi(\bar{\bar{x}}^{t,1}),\dots,\pi(\bar{\bar{x}}^{t,k'}),1-\pi(\bar{\bar{x}}^{t,1}),\dots,1-\pi(\bar{\bar{x}}^{t,k'})\right),\left(1-\bar{\bar{y}}^{t,1},\dots,1-\bar{\bar{y}}^{t,k'},1/2,\dots,1/2\right)\bigg\rangle\\ 
&-\bigg\langle\left(\pi'(\bar{\bar{x}}^{t,1}),\dots,\pi'(\bar{\bar{x}}^{t,k'}),1-\pi'(\bar{\bar{x}}^{t,1}),\dots,1-\pi'(\bar{\bar{x}}^{t,k})\right),\left(1-\bar{\bar{y}}^{t,1},\dots,1-\bar{\bar{y}}^{t,k'},1/2,\dots,1/2\right)\bigg\rangle\bigg]\\
&= 2\left[\E\limits_{h\sim\pi}\left[\langle a^{h},\ell^t\rangle\right] - \E\limits_{h'\sim\pi'}\left[\langle a^{h'},\ell^t\rangle\right]\right].
\end{align*}

Where the last transition stems from the linearity of Error$(\cdot,\bar{\bar{x}}^t,\bar{\bar{y}}^t)$. This concludes the proof.
\end{proof}

We are now ready to prove Theorem \ref{thm:reduction}.
\begin{proof}[Proof of Theorem \ref{thm:reduction}]
We can see that
\begin{align*}
&\sum_{t=1}^T L_{C,\rho^t}(\pi^t,\bar{x}^t,\bar{y}^t) - \min_{\pi^*\in V}\sum_{t=1}^T L_{C,\rho^t}(\pi^*,\bar{x}^t,\bar{y}^t)&&\\
&\leq \sum_{t=1}^T L_{C,\rho^t}(\pi^t,\bar{x}^t,\bar{y}^t) - \min_{\pi^*\in\Delta\cH} \sum_{t=1}^T L_{C,\rho^t}(\pi^*,\bar{x}^t,\bar{y}^t)\quad\quad\quad\quad\quad&&\text{($V \subseteq \Delta\cH)$}\\
&= \sum_{t=1}^T L_{C,\rho^t}(\pi^t,\bar{x}^t,\bar{y}^t) - \min_{\pi^*\in\cH} \sum_{t=1}^T L_{C,\rho^t}(\pi^*,\bar{x}^t,\bar{y}^t)&&\text{(Linearity of $L_{C,\rho^t}(\cdot,\bar{x}^t,\bar{y}^t$)}\\
&=2\left[\sum_{t=1}^T\E\limits_{h^t\sim\pi^t}\left[\langle a^{h^t},\ell^t\rangle\right] -\min_{\pi^*\in\cH} \sum_{t=1}^T \E\limits_{h^*\sim\pi^*}\left[\langle a^{h^*},\ell^t\rangle\right]\right].&&\text{(Lemma \ref{lem:reduction1}+Lemma \ref{lem:reduction2})}\\
&=2(2k+4C)R^{\cA,T,\cH}&&\text{($\forall t\in[T]:\left\vert\left\langle\ell^t,a^t\right\rangle\right\vert\in[0,2k+4C]$)}.
\end{align*}
Which concludes the proof.
\end{proof}

Note that the guarantee of Theorem \ref{thm:reduction} holds when competing with the set of all possibly randomized policies $\Delta\cH$ over the base class, instead of only with respect to the best classifier in $\cH$.

%% file: algorithms.tex
\section{Multi-Criteria No Regret Guarantees} \label{sec:algorithms}

In this section, we will see how the guarantees established in Section \ref{sec:reduction} can be leveraged to provide multi-criteria no regret guarantees, simultaneously for accuracy and for fairness. We begin by establishing an equivalence between auditing by panels and auditing by instance-specific ``representative'' auditors, which will be useful in our analysis.

\subsection{From Panels to ``Representative'' Auditors}\label{subsec:equivalence}

Here, we give a reduction from auditing by panels to auditing by instance-specific single auditors. In particular, we prove that the feedback given by any panel can be viewed as equivalent to the decisions of single, ``representative'', auditors from the panel, where the identity of the relevant auditor is determined only as a function of the specific pair $(x, x')$ in question. 

We observe that when it comes to a single pair $(x, x')$, we can order auditors by their ``strictness'' on this pair, as measured by $d(x, x')$. However  it is not possible in general to  order or compare the level of ``strictness'' of different auditors beyond a single pair, as some may be stricter than others on some pairs, but have the opposite relation on others. For illustration, consider the following example: let
$\cX = \{x^1,x^2,x^3\}, \cJ=\{j^1,j^2\}$ and assume that $d^{j^1}(x^1,x^2) > d^{j^2}(x^1,x^2)$, and
$d^{j^1}(x^2,x^3) < d^{j^2}(x^2,x^3)$. In the context of this example, asking who is stricter or who is more lenient among the auditors, in an absolute sense, is undefined. 

However, as we restrict the attention to a single pair $(x,x')$, such a task becomes feasible. Namely, in spite of the fact that we do not have access to auditors' underlying distance measures (we only observe feedback regarding violations), we know that there is an implicit ordering among the auditors' level of strictness with respect to that specific pair. The idea is to then utilize this (implicit) ordering to argue that a panel's judgements with respect to this pair are in fact equivalent to the judgements of a specific single auditor from the panel, which can be viewed as a ``representative auditor''. We formalize the argument in Lemma \ref{lem:equivalence}.

\begin{lemma}\label{lem:equivalence}
Let $(x,x')\in\cX^2$, $(j^1,...,j^m)\in\cJ^m$. Then, there exists an index $s=s_{x,x'}(j^1,\dots,j^m)\in[m]$ such that the following are equivalent, for all $\pi\in\Delta\cH$:
\begin{enumerate}
    \item $\pi$ has an $(\alpha,\gamma)$-violation on $(x,x')$ with respect to panel $\bar{j}^{\alpha,\gamma}_{j^1,...,j^m}$.
    \item $\pi$ has an $\alpha$-violation on $(x,x')$ with respect to auditor $j^{s}$.
\end{enumerate}
\end{lemma}
The crucial aspect of this lemma is that the index  $s_{x,x'}(j^1,\dots,j^m)$ of the ``pivotal'' auditor is defined independently of $\pi$.

\begin{proof}[Proof of Lemma \ref{lem:equivalence}]
Fix $(x,x')$. Then, we can define an ordering of $(j^1,...,j^m)$ according to their (underlying) distances on $(x,x')$,
\begin{equation}\label{eq:ordering}
d^{j^{i_1}}(x,x')\leq\dots\leq d^{j^{i_m}}(x,x').
\end{equation}

Then, set
\begin{equation}\label{eq:median}
s := s_{x,x'}(j^1,...,j^m) = i_{\left\lceil\gamma m\right\rceil}.
\end{equation}

Note that $s$ in eq. (\ref{eq:median}) is well-defined, since $\gamma \leq 1$.

We also note that, using the ordering defined in eq. (\ref{eq:ordering}), for any $r\in[m]$, 
\begin{equation}\label{eq:stricter}
\pi(x)-\pi(x') > d^{j^{i_r}}(x,x') + \alpha \implies \forall r'\leq r:\pi(x)-\pi(x') > d^{j^{i_{r'}}}(x,x') + \alpha.    
\end{equation}

Hence, when considering a random variable indicating an $(\alpha,\gamma)$-violation on $(x,x')$ with respect to panel $\bar{j}$, we know that
\begin{align*}
&\mathbbm{1}\left[\left[\frac{1}{m}\sum_{i=1}^m \mathbbm{1}\left[\pi(x)-\pi(x') - d^{j^i}(x,x') > \alpha\right]\right]\geq\gamma\right]&&\\ 
&= \mathbbm{1}\left[\left[\frac{1}{m}\sum_{l=1}^s\mathbbm{1}\left[\pi(x)-\pi(x') - d^{j^{i_l}}(x,x') > \alpha\right]\right]\geq\gamma\right]&&\text{(Eq. \ref{eq:median} and Eq. \ref{eq:stricter})}\\
&= \mathbbm{1}\left[\pi(x)-\pi(x') - d^{j^s}(x,x') > \alpha\right]&&\text{(Eq. \ref{eq:median})},
\end{align*}

which is equivalent to indicating an $\alpha$-violation on $(x,x')$ with respect to auditor $j^s$.
This concludes the proof.
\end{proof}

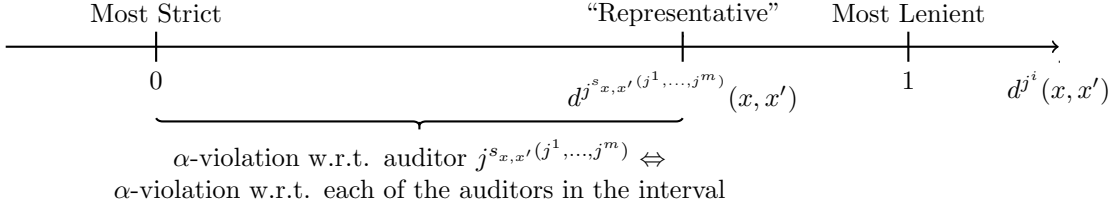
\begin{figure}[t]
\centering
\begin{tikzpicture}
\draw[thick,->] (0,0) -- (14,0); 
\draw[thick,-] (14,-0.2) -- (14,-0.2) node[anchor=north]{$d^{j^i}(x,x')$};
\draw[thick,-] (9,0.2) -- (9,-0.2) node[anchor=north]{$d^{j^{s_{x,x'}(j^1,\dots,j^m)}}(x,x')$};
\draw[thick,-] (2,0.2) node[anchor=south]{Most Strict} -- (2,-0.2) node[anchor=north]{$0$};
\draw[thick,-] (9,0.14) node[anchor=south]{``Representative''} -- (9,0.14);
\draw[thick,-] (12,0.2) node[anchor=south]{Most Lenient} -- (12,-0.2) node[anchor=north]{$1$};
\draw [thick,decoration={brace}, decorate] (9,-0.95) -- (2,-0.95);
\draw [thick,-] (5.5,-1.1) node[anchor=north, text width=10cm,align=center]{$\alpha$-violation w.r.t. auditor $j^{s_{x,x'}(j^1,\dots,j^m)}$ $\Leftrightarrow$ \\ $\alpha$-violation w.r.t. each of the auditors in the interval} -- (5.5,-1);
\end{tikzpicture}
\caption{An illustration of an ordering of a panel of auditors $(j^1,\dots,j^m)$ according to their (implicit) distances on $(x,x')$. ${j^{s_{x,x'}(j^1,\dots,j^m)}}$ denotes the auditor who is in the $\lceil\gamma m\rceil$ position in this ordering, which can also be viewed as having the ``swing vote'' with respect to deciding an $(\alpha,\gamma)$-violation in this instance.
} \label{}
\end{figure}

\subsection{Simultaneous Accuracy and Fairness Guarantees}\label{subsec:simultaneous}

Next, we will see how the guarantees established in Section \ref{sec:reduction}, along with the reduction to ``representative'' auditors of Section \ref{subsec:equivalence}, allow for providing simultaneous guarantees for each of accuracy and fairness. We begin by defining the comparator set as all policies in $\cH$ that are, for every round $t\in[T]$, $(\alpha,\gamma)$-fair on the arriving individuals of the round $\bar{x}^t$, with respect to the realized panel in that round $\bar{j}^t$. Note that this set is only defined in hindsight.

\begin{definition}[$(\alpha,\gamma)$-fair policies]\label{def:comparator-set}
Let $\alpha\geq 0$, $0\leq \gamma \leq 1$, $m\in\mathbb{N}\setminus\{0\}$. We denote the set of all $(\alpha,\gamma)$-fair policies with respect to all of the rounds in the run of the algorithm as
\begin{align*}
Q_{\alpha,\gamma} := \left\{\pi\in\Delta\cH : \forall t\in[T],~ \bar{j}^{t,\alpha,\gamma}_{j^{t,1},\dots,j^{t,m}}(\pi,\bar{x}^t)=(v,v)\right\}.
\end{align*}
\end{definition}

Next, we show how the Lagrangian regret guarantee established in Theorem \ref{thm:reduction} can be utilized to provide simultaneous guarantees for accuracy and fairness, when compared with the most accurate policy in $Q_{\alpha-\epsilon,\gamma}$. Note, in particular, that by setting $Q_{\alpha-\epsilon,\gamma}$ as the comparator set, we will be able to upper bound the number of rounds in which an $(\alpha,\gamma)$-violation has occurred.

\begin{lemma}\label{lem:joint-loss}
For any $\epsilon\in[0,\alpha]$,
\begin{align*}
& C\epsilon\sum_{t=1}^T Unfair^{\alpha,\gamma}(\pi^t,\bar{x}^t,\bar{j}^t) + Regret^{err}(\cA,T,Q_{\alpha-\epsilon,\gamma})\\
&\leq \sum_{t=1}^T L_{C,\rho^t}(\pi^t,\bar{x}^t,\bar{y}^t) - \min_{\pi^*\in Q_{\alpha-\epsilon,\gamma}}\sum_{t=1}^T L_{C,\rho^t}(\pi^*,\bar{x}^t,\bar{y}^t).
\end{align*}
\end{lemma}

\paragraph{High-level proof idea} It is sufficient to prove that for every $\pi^*\in Q_{\alpha-\epsilon,\gamma}$, if we set $\pi^*$ as the comparator (instead of taking the minima for each the error regret and the Lagrangian regret), the inequality holds. We will hence fix such $\pi^*\in Q_{\alpha-\epsilon,\gamma}$, and using Definition \ref{def:Lagrangian}, see that the sums of the $Error$ terms in both sides of the inequality cancel out. We will then divide the analysis to two cases: rounds on which no $(\alpha,\gamma)$-violation was detected, and rounds where such a violation was detected. For the first type, the equality holds, since by definition of the panel, $\rho^{t,1}=\rho^{t,2}$. For the second type, the left hand side of the equality is simply $C\epsilon$. As for the right hand side, we first use Lemma \ref{lem:equivalence} from Section \ref{subsec:equivalence} to move from panels to carefully defined ``representative'' single auditors, then argue that the right hand side of the equality is at least $C\epsilon$, since $\pi^*$ is guaranteed to have no $(\alpha-\epsilon)$-violations with respect to any of the ``representative'' auditors of the panels in the interaction.

\begin{proof}[Proof of Lemma \ref{lem:joint-loss}]

To prove the lemma, it is sufficient to prove that for every $\pi^*\in Q_{\alpha-\epsilon,\gamma}$,
\begin{align*}
& C\epsilon\sum_{t=1}^T Unfair^{\alpha,\gamma}(\pi^t,\bar{x}^t,\bar{j}^t) + \sum_{t=1}^T Error(\pi^t,\bar{x}^t,\bar{y}^t) - \sum_{t=1}^T Error(\pi^*,\bar{x}^t,\bar{y}^t)\\
&\leq \sum_{t=1}^T L_{C,\rho^t}(\pi^t,\bar{x}^t,\bar{y}^t) - \sum_{t=1}^T L_{C,\rho^t}(\pi^*,\bar{x}^t,\bar{y}^t).
\end{align*}

Which, using Definition \ref{def:Lagrangian}, is equivalent to proving that
\[
C\epsilon\sum_{t=1}^T Unfair^{\alpha,\gamma}(\pi^t,\bar{x}^t,\bar{j}^t) \leq \sum_{t=1}^T C\cdot\left[\pi^t(\rho^{t,1})-\pi^t(\rho^{t,2})\right] - \sum_{t=1}^T C\cdot\left[\pi^*(\rho^{t,1})-\pi^*(\rho^{t,2})\right].
\]

We consider two cases:
\begin{enumerate}
    \item For rounds $t$ where the panel $\bar{j}^t$ did not detect any $(\alpha,\gamma)$-fairness violations, the left hand side of the inequality is 0, and so is the right hand side, since $\rho^{t,1} = \rho^{t,2}$.
    \item For rounds $t$ where the panel $\bar{j}^t$ detected an $(\alpha,\gamma)$-violation, the left hand side is equal to $C\epsilon$, and the right hand side is at least $C\epsilon$, since, using Lemma \ref{lem:equivalence} and Definition \ref{def:comparator-set}, we know that
    \begin{equation} \label{eq:alpha-violation}
    \pi^t(\rho^{t,1})-\pi^t(\rho^{t,2}) > d^{s_{\rho^{t,1},\rho^{t,2}}(j^{t,1},\dots,j^{t,m})}(\rho^{t,1},\rho^{t,2}) + \alpha
    \end{equation}
    And
    \begin{equation} \label{eq:no-violation}
    -(\pi^*(\rho^{t,1})-\pi^*(\rho^{t,2}))\geq \epsilon-\alpha - d^{s_{\rho^{t,1},\rho^{t,2}}(j^{t,1},\dots,j^{t,m})}(\rho^{t,1},\rho^{t,2})
    \end{equation}
    Hence, combining Equation \ref{eq:alpha-violation} and Equation \ref{eq:no-violation}, we get
    \begin{align*}
    &\pi^t(\rho^{t,1})-\pi^t(\rho^{t,2})-(\pi^*(\rho^{t,1})-\pi^*(\rho^{t,2}))\\
    &\geq d^{s_{\rho^{t,1},\rho^{t,2}}(j^{t,1},\dots,j^{t,m})}(\rho^{t,1},\rho^{t,2}) + \alpha + \epsilon-\alpha - d^{s_{\rho^{t,1},\rho^{t,2}}(j^{t,1},\dots,j^{t,m})}(\rho^{t,1},\rho^{t,2})\\
    &\geq \epsilon.
    \end{align*}
    
\end{enumerate}

The lemma hence follows.
\end{proof}

\subsection{Bounds for Specific Algorithms} \label{subsec:algorithms}

In this section, we present two algorithms for the contextual combinatorial semi-bandit setting (Algorithm \ref{alg:ccsb}), and show how they can be leveraged to establish accuracy and fairness guarantees in the setting of individually fair online learning with one-sided feedback (Algorithm \ref{alg:apple-tasting}). In the following, we use the notation $\Vert \ell^t\Vert_* =\max_{a\in A^t}\left\vert \left\langle \ell^t,a \right\rangle\right\vert$, and use $\tilde{O}$ to hide logarithmic factors.

\subsubsection{Exp2}
We begin by presenting the Exp2 algorithm \citep{BubeckCK12, DaniHK07,Cesa-BianchiL09} and showing how it can be adapted to our setting.\footnote{The contextual combinatorial semi-bandit setting considered in this paper subsumes the standard contextual $k$-armed bandit setting. To see this, consider the case where $A^t = A = \{a^{t,i} = (\mathbbm{1}[i=1],\dots,\mathbbm{1}[i=k]):i\in[k]\}$. Naively applying the classical EXP4 algorithm for contextual bandits in the combinatorial semi-bandit setting would result in a regret bound of $O(\sqrt{\vert\cH\vert T})$, whose square root dependence on $\vert\cH\vert$ we prefer to avoid.}
Exp2 is an adaptation of the classical exponential weights algorithm \citep{LittlestoneW94,AuerCFS95,Vovk90,FreundS97,cesa}, which, in order to cope with the semi-bandit nature of the online setting, leverages the linear structure of the loss functions in order to share information regarding the observed feedback between all experts (hypotheses in $\cH$). Such information sharing is then utilized in decreasing the variance in the formed loss estimators, resulting in a regret rate that depends only logarithmically on $\vert \cH \vert$.

\begin{theorem}[via \citet{BubeckCK12}]\label{thm:exp2}
The expected regret of Exp2 in the contextual combinatorial semi-bandit setting, against any adaptively and adversarially chosen sequence of contexts and linear losses such that $\Vert \ell^t\Vert_*\leq 1$, is at most:
\[
Regret(T) \leq O\left(\sqrt{kT\log\vert\cH\vert}\right).
\]
\end{theorem}

Next, we show how, when leveraging our reduction as described in Section \ref{sec:reduction}, Exp2 can be utilized to provide multi-criteria guarantees, simultaneously for accuracy and fairness.

\begin{theorem}\label{thm:exp2-conclusion}
In the setting of individually fair online learning with one-sided feedback (Algorithm \ref{alg:apple-tasting}), running Exp2 for contextual combinatorial semi-bandits (Algorithm \ref{alg:ccsb}) while using the sequence $(a^t,\ell^t)_{t=1}^T$ generated by the reduction in Algorithm \ref{alg:reduction} (when invoked each round using $\bar{x}^t$, $\bar{y}^t$, $h^t$, $\rho^t$, and $C=T^\frac{1}{5}$), yields the following guarantees, for any $\epsilon\in[0,\alpha]$, simultaneously:

\begin{enumerate}
    \item \textbf{Accuracy:} $Regret^{err}(\text{Exp2},T,Q_{\alpha-\epsilon,\gamma}) \leq O\left(k^\frac{3}{2}T^\frac{4}{5}\log\vert\cH\vert^\frac{1}{2}\right)$.
    \item \textbf{Fairness:} 
    $\sum_{t=1}^T Unfair^{\alpha,\gamma}(\pi^t,\bar{x}^t,\bar{j}^t) \leq O\left(\frac{1}{\epsilon}k^\frac{3}{2}T^\frac{4}{5}\log\vert\cH\vert^\frac{1}{2}\right)$.
\end{enumerate}
\end{theorem}

\begin{proof}[Proof of Theorem \ref{thm:exp2-conclusion}]
Combining Theorems \ref{thm:reduction}, \ref{thm:exp2}, we know that
\begin{align*}
\sum_{t=1}^T L_{C,\rho^t}(\pi^t,\bar{x}^t,\bar{y}^t) - \min_{\pi^*\in Q_{\alpha-\epsilon,\gamma}}\sum_{t=1}^T L_{C,\rho^t}(\pi^*,\bar{x}^t,\bar{y}^t) \leq  O\left((2k+4C)^{\frac{3}{2}}\sqrt{T\log\vert\cH\vert}\right).
\end{align*}

Setting $C=T^\frac{1}{5}$, and using Lemma \ref{lem:joint-loss}, we get
\begin{align*}
Regret^{err}(\text{Exp2},T,Q_{\alpha-\epsilon,\gamma})
&\leq O\left((2k+4C)^{\frac{3}{2}}\sqrt{T\log\vert\cH\vert}\right) - C\epsilon\sum_{t=1}^T Unfair^{\alpha,\gamma}(\pi^t,\bar{x}^t,\bar{j}^t)\\
&\leq O\left((2k+4C)^{\frac{3}{2}}\sqrt{T\log\vert\cH\vert}\right)\\
&\leq O\left(k^\frac{3}{2}T^\frac{4}{5}\log\vert\cH\vert^\frac{1}{2}\right).
\end{align*}

And,
\begin{align*}
\sum_{t=1}^T Unfair^{\alpha,\gamma}(\pi^t,\bar{x}^t,\bar{j}^t) &\leq \frac{1}{C\epsilon}\left[O\left((2k+4C)^{\frac{3}{2}}\sqrt{T\log\vert\cH\vert}\right) - Regret^{err}(\text{Exp2},T,Q_{\alpha-\epsilon,\gamma})\right]\\
&\leq \frac{1}{C\epsilon}\left[O\left((2k+4C)^{\frac{3}{2}}\sqrt{T\log\vert\cH\vert}\right) + kT\right]\\
&\leq O\left(\frac{1}{\epsilon}k^\frac{3}{2}T^\frac{4}{5}\log\vert\cH\vert^\frac{1}{2}\right).
\end{align*}
\end{proof}

The guarantees of Theorem \ref{thm:exp2-conclusion} can be interpreted as follows: accuracy-wise, the resulting algorithm is competitive with the performance of the most accurate policy that is fair (i.e. in $Q_{\alpha-\epsilon,\gamma}$). Fairness-wise, the number of rounds in which there exist (one or more) fairness violations, is sub-linear.

While presenting statistically optimal performance in terms of its dependence on the number of rounds and the cardinality of the hypothesis class, Exp2 is in  general computationally inefficient, with runtime and space requirements that are linear in $\vert \cH \vert$, which is prohibitive for large hypothesis classes. We hence next propose an oracle-efficient algorithm, based on a combinatorial semi-bandit variant of the classical Follow-The-Perturbed-Leader (FTPL) algorithm \citep{KalaiV05, Hannan}.

\subsubsection{Context-Semi-Bandit-FTPL}
We next present an oracle-efficient algorithm, Context-Semi-Bandit-FTPL \citep{SyrgkanisKS16}, and construct an adaptation of it for the setting of individually fair online learning with one-side feedback. Broadly speaking, FTPL-style algorithms' approach is to solve ``perturbed'' optimization problems. Namely, at each round, the set of data samples observed so far is augmented, using carefully drawn additional noisy samples. Then, the resulting ``perturbed'' optimization problem over the augmented sample set is solved. In doing so, the procedure carefully combines the objectives of stability and error minimization, in order to provide no regret guarantees.

In order to construct an \emph{efficient} implementation of this approach in the setting of contextual combinatorial semi-bandit, Context-Semi-Bandit-FTPL assumes access to two key components: an offline optimization oracle for the base class $\cH$, and a small separator set for $\cH$. The optimization oracle assumption can be viewed equivalently as assuming access to a weighted ERM oracle for $\cH$. We next describe the small separator set assumption.

\begin{definition}[Separator set]
We say $S\subseteq\cX$ is a separator set for a class $\cH:\cX\rightarrow\{0,1\}$, if for any two distinct hypotheses $h, h'\in
\cH$, there exists $x\in S$ such that $h(x)\neq h'(x)$.
\end{definition}

\begin{remark}
Classes for which small separator sets are known include conjunctions, disjunctions, parities, decision lists, discretized linear classifiers. Please see more elaborate discussions in \citet{SyrgkanisKS16} and \citet{Neel0W19}.
\end{remark}
For the following theorem, it is assumed that Context-Semi-Bandit-FTPL has access to a (pre-computed) separator set $S$ of size $s$ for the class $\cH$, and access to an (offline) optimization oracle for $\cH$.

\begin{theorem}[via \citet{SyrgkanisKS16}] \label{thm:csb-ftpl}
The expected regret of Context-Semi-Bandit-FTPL in the contextual combinatorial semi-bandit setting, against any adaptively and adversarially chosen sequence of contexts and linear non-negative losses such that $\Vert \ell^t\Vert_*\leq 1$, is at most:
\[
Regret(T) \leq O\left(k^{\frac{7}{4}}s^{\frac{3}{4}}T^{\frac{2}{3}}\log{\vert\cH\vert}^{\frac{1}{2}}\right).
\]
\end{theorem}

We note that Context-Semi-Bandit-FTPL does not, at any point, maintain its deployed distribution over the class $\cH$ explicitly. Instead, on each round, it ``samples'' a hypothesis according to such (implicit) distribution --- where the process of perturbing then solving described above can equivalently be seen as sampling a single hypothesis from such underlying distribution over $\cH$.

\paragraph{Resampling-based adaptation} For our purposes, however, we will have to adapt the implementation of Context-Semi-Bandit-FTPL so that the process of sampling the hypothesis at each round is repeated, and we are able to form an accurate enough empirical estimate of the implicit distribution. This is required for two reasons: first, as we wish to compete with the best fair \emph{policy} in $\Delta\cH$, rather than only with the best fair classifier in $\cH$ (we elaborate on this point in Lemma \ref{lem:gap}). Second, as it is observed in general (see, e.g. the discussion in \citet{NeuB13}), the specific weights this implicit distribution places on each of $h\in\cH$ cannot be expressed in closed-form.

We therefore next construct an adaptation we term Context-Semi-Bandit-FTPL-With-Resampling, which is based on resampling the hypothesis $R$ times and deploying the empirical estimate $\hat{\pi}^t$ of the (implicit) underlying distribution $\pi^t$. This adaptation is summarized in Algorithm \ref{alg:Adapted-Apple-Tasting} and Algorithm \ref{alg:CSB-FTPL-With-Resampling} below, and yields the following guarantee.

\begin{theorem}\label{thm:FTPL-conclusion}
In the setting of individually fair online learning with one-sided feedback (Algorithm \ref{alg:apple-tasting}), running Context-Semi-Bandit-FTPL-With-Resampling for contextual combinatorial semi-bandit (Algorithm \ref{alg:CSB-FTPL-With-Resampling}) as specified in Algorithm \ref{alg:Adapted-Apple-Tasting}, with $R=T$, and using the sequence $(\ell^t,a^t)_{t=1}^T$ generated by the reduction in Algorithm \ref{alg:reduction} (when invoked on each round using $\bar{x}^t$, $\bar{y}^t$, $\hat{h}^t$, $\hat{\rho}^t$, and $C=T^\frac{4}{45}$), yields, with probability $1-\delta$, the following guarantees, for any $\epsilon\in[0,\alpha]$, simultaneously:

\begin{enumerate}
    \item \textbf{Accuracy:} $Regret^{err}(\text{CSB-FTPL-WR},T,Q_{\alpha-\epsilon,\gamma}) \leq \tilde{O}\left(k^\frac{11}{4}s^\frac{3}{4}T^\frac{41}{45}\log\vert\cH\vert^\frac{1}{2}\right)$.
    \item \textbf{Fairness:} $\sum_{t=1}^T Unfair^{\alpha,\gamma}(\hat{\pi}^t,\bar{x}^t,\bar{j}^t) \leq \tilde{O}\left(\frac{1}{\epsilon}k^\frac{11}{4}s^\frac{3}{4}T^\frac{41}{45}\log\vert\cH\vert^\frac{1}{2}\right)$.
\end{enumerate}
\end{theorem}

We next describe the adaptation of Context-Semi-Bandit-FTPL \citep{SyrgkanisKS16} to our setting. Context-Semi-Bandit-FTPL relies on access to an optimization oracle for the corresponding (offline) problem. We elaborate on the exact implementation of this oracle in our setting next. 

\begin{definition}[Optimization oracle]
Context-Semi-Bandit-FTPL assumes access to an oracle of the form
\[
M((\bar{\bar{x}}^t)_{t=1}^N,(\hat{\ell}^t)_{t=1}^N) = \argmin\limits_{h\in\cH} L(h,(\bar{\bar{x}}^t,\hat{\ell}^t)),
\]
where $\hat{\ell}^t$ denotes the loss estimates held by Context-Semi-Bandit-FTPL for round $t$, and $L$ denotes the cumulative loss, over linear loss functions of the form $f^t(a) =\left\langle a,\ell\right\rangle$, where $\ell$ is a non-negative vector. In our construction, this is equivalent to
\begin{align*}
&\argmin\limits_{h\in\cH} L(h^t,(\bar{\bar{x}}^t,\hat{\ell}^t))\\
&:= \argmin\limits_{h\in\cH}\sum_{t=1}^N\left\langle a^t_h,\hat{\ell}^t \right\rangle &&\text{(Definition of L)}\\
&= \argmin\limits_{h\in\cH}\sum_{t=1}^N\sum_{i=1}^{k+2C} h(\bar{\bar{x}}^{t,i})\cdot \hat{\ell}^{t,i} + (1-h(\bar{\bar{x}}^{t,i}))\cdot \frac{1}{2} &&\text{(Algorithm  \ref{alg:reduction})}\\
&= \argmin\limits_{h\in\cH}\sum_{t=1}^N\sum_{i=1}^{k+2C} h(\bar{\bar{x}}^{t,i})\cdot (\hat{\ell}^{t,i}-\frac{1}{2}) &&\text{(Subtraction of constant)}
.
\end{align*}
\end{definition}

Context-Semi-Bandit-FTPL operates by, at each round, first sampling a set of ``fake'' samples $z^t$, that is added to the history of observed contexts and losses by the beginning of round $t$, denoted by $H^t$. The algorithm then invokes the optimization oracle on the extended set $z^t \cup H^t$, and deploys $h^t\in\cH$ that is returned by the oracle.  

Equivalently, this process can be seen as the learner, at the beginning of each round $t$, (implicitly) deploying a distribution over hypotheses from the base class $\cH$, denoted by $\pi^t$, then sampling and deploying a single hypothesis $h^t\sim\pi^t$. As it is observed in general (see, e.g., \citet{NeuB13}), the specific weights this implicit distribution places on each of $h\in\cH$ on any given round cannot be expressed in closed-form. Instead, FTPL-based algorithms rely on having sampling access to actions from the distribution in obtaining expected no regret guarantees. 

For our purposes, however, such a method of assessing the loss on realized (single) hypotheses $h^t\sim\pi^t$ could be problematic, since we rely on the panel $\bar{j}^t$ reporting its feedback upon observing the actual distribution $\pi^t$. Querying the panel instead using realizations $h^t\sim\pi^t$ could lead to an over-estimation of the unfairness loss, as we demonstrate next.

\begin{lemma}\label{lem:gap}
There exist $\alpha,\gamma, m, k > 0$, $\cH:\cX\rightarrow\{0,1\}$, $\bar{x}\in\cX^k$, $\bar{j}:\cX^k\rightarrow\cX^2$, and $\pi\in\Delta\cH$ for which, simultaneously,
\begin{enumerate}
    \item $\E\limits_{h\sim\pi}\left[unfair^{\alpha,\gamma}(h,\bar{x},\bar{j})\right] = 1$.
    \item $unfair^{\alpha,\gamma}(\pi,\bar{x},\bar{j}) = 0$.
\end{enumerate}
\end{lemma}

We defer the proof of Lemma \ref{lem:gap} to Appendix \ref{app:algorithms}.

We therefore adapt Context-Semi-Bandit-FTPL to our setting by adding a resampling process at each iteration of the algorithm. Our approach is similar in spirit to the resampling-based approach in \citet{BechavodJW20} (which offer an adaptation for the full information variant of the algorithm), however, unlike their suggested scheme, which requires further restricting the power of the adversary to, at each round $t$, not depend on the policy $\pi^t$ deployed by the learner (instead, they only allow dependence on the history of the interaction until round $t-2$), the adaptation we next propose would not require such a relaxation.

We next abstract out the implementation details of the original Context-Semi-Bandit-FTPL that remain unchanged (namely, the addition of ``fake'' samples, and solving of the resulting optimization problem at the beginning of each round, and the loss estimation process at the end of it), to focus on the adaptation. 

Our adaptation will work as follows: the learner initializes Context-Semi-Bandit-FTPL-With-Resampling with a pre-computed separator set $S$ for $\cH$. Then, at each round $t$, the learner (implicitly) deploys $\pi^t$ according to Context-Semi-Bandit-FTPL-With-Resampling. The environment then selects individuals $\bar{x}^t$ and their labels $\bar{y}^t$, only revealing $\bar{x}^t$ to the learner. The environment proceeds to select a panel of auditors $(j^{t,1},\dots,j^{t,m})$. The learner invokes Context-Semi-Bandit-FTPL-With-Resampling and receives an estimated policy $\hat{\pi}^t$, and a realized predictor $\hat{h}^t$ sampled from $\hat{\pi}^t$. The learner then predicts the arriving individuals $\bar{x}^t$ using $\hat{h}^t$, only observing feedback on positively labelled instances. The panel then reports its feedback $\hat{\rho}^t$ on $(\hat{\pi}^t,\bar{x}^t)$. The learner invokes the reduction (Algorithm \ref{alg:reduction}), using $\bar{x}^t$, $\bar{y}^t$, $\hat{h}^t$, $\hat{\rho}^t$, and $C$, and receives $(\ell^t,a^t)$. The learner updates Context-Semi-Bandit-FTPL-With-Resampling with $(\ell^t,a^t)$ and lets it finish the loss estimation process and deploy the policy for the next round. Finally, the learner suffers misclassification loss with respect to $\hat{h}^t$, and unfairness loss with respect to $\hat{\pi}^t$. The interaction is summarized in Algorithm \ref{alg:Adapted-Apple-Tasting}. 

As for the resampling process we add to the original Context-Semi-Bandit-FTPL: at each round we define ``sampling from $\cD^t$'' to refer to the process of first sampling the additional ``fake'' samples to be added, and then solving the resulting optimization problem over the original and the ``fake'' samples, to produce a predictor $h^{t,r}$. We repeat this process $R$ times, to produce an empirical distribution $\hat{\pi}^t$, and select a single predictor $\hat{h}^t$ from it, which are reported to the learner. Once receiving back $(\ell^t,a^t)$ from the learner, Context-Semi-Bandit-FTPL-With-Resampling proceeds to perform loss estimation, as well as selecting the next policy, in a similar fashion to the original version of Context-Semi-Bandit-FTPL. This adaptation is summarized in Algorithm \ref{alg:CSB-FTPL-With-Resampling}.

\begin{algorithm}[t]
\caption{Utilization of Context-Semi-Bandit-FTPL}\label{alg:Adapted-Apple-Tasting}
\SetAlgoLined
\textbf{Parameters:} Class of predictors $\cH$, number of rounds $T$, separator set $S$, parameters $\omega$, $L$\;
Initialize Context-Semi-Bandit-FTPL-With-Resampling$(S,\omega,L)$\;
Learner deploys $\pi^1\in\Delta\cH$ according to Context-Semi-Bandit-FTPL-With-Resampling\;
\For{$t = 1, \ldots, T$}{
Environment selects individuals $\bar{x}^t \in\cX^k$, and labels $\bar{y}^t\in\cY^k$, learner only observes $\bar{x}^t$\;
Environment selects panel of auditors $(j^{t,1},\dots,j^{t,m})\in\cJ^m$\;
$(\hat{\pi}^t,\hat{h}^t)$ = Context-Semi-Bandit-FTPL-With-Resampling$(\bar{x}^t,\omega,L)$\;
Learner predicts $\hat{y}^{t,i} = h^t(\bar{x}^{t,i})$ for each $i \in [k]$, observes $\bar{y}^{t,i}$ iff $\hat{y}^{t,i} = 1$\;
Panel reports its feedback $\rho^t = \bar{j}^{t,\alpha,\gamma}_{j^1,\dots,j^m}(\hat{\pi}^t,\bar{x}^t)$\;
$(\ell^t$,$a^t)$ = Reduction$(\bar{x}^t, \bar{y}^t,\hat{h}^t,\rho^t,C)$\;
Update Context-Semi-Bandit-FTPL-With-Resampling with $(\ell^t$,$a^t)$\;
Learner suffers misclassification loss $Error(\hat{h}^t,\bar{x}^t,\bar{y}^t)$ (not necessarily observed by learner)\;
Learner suffers unfairness loss $Unfair(\hat{\pi}^t,\bar{x}^t,\bar{j}^t)$\;
Learner deploys $\pi^{t+1}\in\Delta\cH$ according to Context-Semi-Bandit-FTPL-With-Resampling\;
}
\end{algorithm}

\begin{algorithm}[ht]
\caption{Context-Semi-Bandit-FTPL-With-Resampling$(S,\omega,L)$}\label{alg:CSB-FTPL-With-Resampling}
\SetAlgoLined
\textbf{Parameters:} Class of predictors $\cH$, number of rounds $T$, optimization oracle $M$, separator set $S$, parameters $\omega$, $L$;\
\For{$t = 1, \ldots, T$}{
    \For{$r = 1, \ldots, R$}{
        Sample predictor $h^{t,r}$ according to $\cD^t$;\
    }
    Set and report $\hat{\pi}^t = \mathbb{U}(h^{t,1},\dots,h^{t,R})$, $\hat{h}^t\sim\hat{\pi}^t$;\
    Receive back $(\ell^t,a^t)$ from reduction;\ 
    Continue as original Context-Semi-Bandit-FTPL;\
}
\end{algorithm}

We note that for the described adaptation, we will next prove accuracy and fairness guarantees for the sequence of estimated policies, $(\hat{\pi}^t)_{t=1}^T$, rather than for the underlying policies $(\pi^t)_{t=1}^T$. One potential issue with this approach is that the Lagrangian loss at each round is defined using the panel's reported pair $\rho^t$, which is assumed to be reported with respect to $\pi^t$. Here, we instead consider the Lagrangian loss using $\hat{\rho}^t$, which is based on the realized estimation $\hat{\pi}^t$. However, this issue can be circumvented with the following observation: on each round, there are $k^2$ options for selecting $\rho^t$, which are simply all pairs in $\bar{x}^t$. We will prove next, that since resampling for $\hat{\pi}^t$ is done after $\bar{x}^t$ is fixed, with high probability, the Lagrangian loss for each of $\pi^t$ and $\hat{\pi}^t$ will take values that are close to each other, when defined using any possible pair $\hat{\rho}^t$ from $\bar{x}^t$. Hence, by allowing the adversary the power to specify $\hat{\rho}^t$ after $\hat{\pi}^t$ is realized, we do not lose too much. We formalize this argument next.

\begin{theorem}\label{thm:reduction-FTPL} In the setting of (adapted) individually fair online learning with one-sided feedback (Algorithm \ref{alg:Adapted-Apple-Tasting}), running Context-Semi-Bandit-FTPL-With-Resampling (Algorithm \ref{alg:CSB-FTPL-With-Resampling}) with $L=T^\frac{1}{3}$, and optimally selected $\omega$, using the sequence $(a^t,\ell^t)_{t=1}^T$ generated by the reduction in Algorithm \ref{alg:reduction} (when invoked every round with $\bar{x}^t$, $\bar{y}^t$, $\hat{h}^t$, $\hat{\rho}^t$, and $C$), yields, with probability $1-\delta$, the following guarantee, for any $U \subseteq \Delta\cH$,
\begin{align*}
\sum_{t=1}^T L_{C,\hat{\rho}^t}(\hat{\pi}^t,\bar{x}^t,\bar{y}^t) - \min_{\pi^*\in U} \sum_{t=1}^T L_{C,\hat{\rho}^t}(\pi^*,\bar{x}^t,\bar{y}^t)&\leq O\left((2k+4C)^{\frac{11}{4}}s^{\frac{3}{4}}T^{\frac{2}{3}}\log{\vert\cH\vert}^{\frac{1}{2}}\right)\\
&+ 2(2k+4C)T\sqrt{\frac{\log\left(\frac{2k T}{\delta}\right)}{2R}}.
\end{align*}
\end{theorem}

In order to prove Theorem \ref{thm:reduction-FTPL}, we will first prove the following lemma, regarding the difference of losses between the underlying $\pi^t$ and the estimated $\hat{\pi}^t$.

\begin{lemma} \label{lem:estimation}
With probability $1-\delta$ (over the draw of $(h^{t,1},\dots,h^{t,R})_{t=1}^T$), for any arbitrary sequence of reported pairs $(\rho^t)_{t=1}^T$ such that $\forall t\in[T], \rho^t\in \left(\bar{x}^t\times\bar{x}^t\right) \cup \{(v,v)\}$,
\[
\sum_{t=1}^T\left\vert\E\limits_{\hat{h}^t\sim\hat{\pi}^t}\left[\langle a^{\hat{h}^t},\ell^t\rangle\right] - \E\limits_{h^t\sim\pi^t}\left[\langle a^{h^t},\ell^t\rangle\right]\right\vert
\leq 2(2k+4C)T\sqrt{\frac{\log\left(\frac{2k T}{\delta}\right)}{2R}}.
\]
\end{lemma}

\begin{proof}
Using Chernoff bound, we can bound the difference in predictions between the underlying and the estimated distributions over base classifiers, for each of the contexts in $\bar{x}^t$, for any round $t$:
\[\forall t\in[T], i\in[k]:
\Pr\left[\left\vert\hat{\pi}^t(\bar{x}^{t,i})-\pi^t(\bar{x}^{t,i})\right\vert \geq \sqrt{\frac{\log\left(\frac{2k T}{\delta}\right)}{2R}}\right]\leq \frac{\delta}{kT}.
\]

Union bounding over all rounds, and each of the contexts in a round, we get that, with probability $1-\delta$,
\[
\forall t\in[T],i\in[k]: \left\vert\hat{\pi}^t(\bar{x}^{t,i})-\pi^t(\bar{x}^{t,i})\right\vert \leq \sqrt{\frac{\log\left(\frac{2k T}{\delta}\right)}{2R}}.
\]

Hence, when considering pairs of individuals, and using triangle inequality, we know that with probability $1-\delta$,
\[
\forall t\in[T],i,j\in[k]: \left\vert\left[\hat{\pi}^t(\bar{x}^{t,i}) - \hat{\pi}^t(\bar{x}^{t,j})\right]-\left[\pi^t(\bar{x}^{t,i})- \pi^t(\bar{x}^{t,j})\right]\right\vert \leq 2\sqrt{\frac{\log\left(\frac{2k T}{\delta}\right)}{2R}}.
\]

Hence, by construction of the losses and actions sequence (using the reduction in Algorithm \ref{alg:reduction} with $\bar{x}^t$, $\bar{y}^t$, $\hat{h}^t$, $\hat{\rho}^t$, and $C$), with probability $1-\delta$,
\[
\forall t\in[T], \hat{\rho}^t\in \left(\bar{x}^t\times\bar{x}^t\right) \cup \{(v,v)\}: \left\vert\E\limits_{\hat{h}^t\sim\hat{\pi}^t}\left[\langle a^{\hat{h}^t},\ell^t\rangle\right] - \E\limits_{h^t\sim\pi^t}\left[\langle a^{h^t},\ell^t\rangle\right]\right\vert \leq 2(2k+4C)\sqrt{\frac{\log\left(\frac{2k T}{\delta}\right)}{2R}}.
\]

Summing over rounds, with probability $1-\delta$, for any arbitrary sequence of reported pairs $(\rho^t)_{t=1}^T$, such that $\forall t\in[T], \rho^t\in \left(\bar{x}^t\times\bar{x}^t\right) \cup \{(v,v)\}$:
\[
\sum_{t=1}^T\left\vert\E\limits_{\hat{h}^t\sim\hat{\pi}^t}\left[\langle a^{\hat{h}^t},\ell^t\rangle\right] - \E\limits_{h^t\sim\pi^t}\left[\langle a^{h^t},\ell^t\rangle\right]\right\vert \leq 2(2k+4C)T\sqrt{\frac{\log\left(\frac{2k T}{\delta}\right)}{2R}}.
\]

Which concludes the proof of the lemma.
\end{proof}

We are now ready to prove the regret bound of Context-Semi-Bandit-FTPL-With-Resampling.

\begin{proof}[Proof of Theorem \ref{thm:reduction-FTPL}]
Using Theorems \ref{thm:csb-ftpl} and \ref{thm:reduction} along with the fact that $\Vert\ell^t\Vert_*\leq 2k+4C$, for any sequence $(\rho^t)_{t=1}^T$ such that $\forall t\in[T], \rho^t\in \left(\bar{x}^t\times\bar{x}^t\right) \cup \{(v,v)\}$,
\[
2\left[\sum_{t=1}^T\E\limits_{h^t\sim\pi^t}\left[\langle a^{h^t},\ell^t\rangle\right] -\min_{\pi^*\in\Delta\cH} \sum_{t=1}^T \E\limits_{h^*\sim\pi^*}\left[\langle a^{h^*},\ell^t\rangle\right]\right] \leq O\left((2k+4C)^{\frac{11}{4}}s^{\frac{3}{4}}T^{\frac{2}{3}}\log{\vert\cH\vert}^{\frac{1}{2}}\right).
\]

Using Lemma \ref{lem:estimation} and the triangle inequality, we conclude that, with probability $1-\delta$,
\begin{align*}
\sum_{t=1}^T L_{C,\hat{\rho}^t}(\hat{\pi}^t,\bar{x}^t,\bar{y}^t) - \min_{\pi^*\in U} \sum_{t=1}^T L_{C,\hat{\rho}^t}(\pi^*,\bar{x}^t,\bar{y}^t) &\leq O\left((2k+4C)^{\frac{11}{4}}s^{\frac{3}{4}}T^{\frac{2}{3}}\log{\vert\cH\vert}^{\frac{1}{2}}\right)\\
&+ 2(2k+4C)T\sqrt{\frac{\log\left(\frac{2k T}{\delta}\right)}{2R}}.
\end{align*}

\end{proof}

We are now ready to prove Theorem \ref{thm:FTPL-conclusion}

\begin{proof}[Proof of Theorem \ref{thm:FTPL-conclusion}]
Using Theorem \ref{thm:reduction-FTPL} 
with $C=T^\frac{4}{45}$, $R=T^\frac{38}{45}$,
we know that, with probability $1-\delta$,
\begin{align*}
\sum_{t=1}^T L_{C,\bar{\rho}^t}(\hat{\pi}^t,\bar{x}^t,\bar{y}^t) - \min_{\pi^*\in Q_{\alpha-\epsilon,\gamma}}\sum_{t=1}^T L_{C,\hat{\rho}^t}(\pi^*,\bar{x}^t,\bar{y}^t) \leq  \tilde{O}\left(k^{\frac{11}{4}}s^{\frac{3}{4}}T^{\frac{41}{45}}\log{\vert\cH\vert}^{\frac{1}{2}}\right).
\end{align*}

Using Lemma \ref{lem:joint-loss}, we get, with probability $1-\delta$,
\begin{align*}
Regret^{err}(\text{CSB-FTPL-WR},T,Q_{\alpha-\epsilon,\gamma})
&\leq \tilde{O}\left(k^\frac{11}{4}s^\frac{3}{4}T^\frac{41}{45}\log\vert\cH\vert^\frac{1}{2}\right) - \sum_{t=1}^T Unfair^{\alpha,\gamma}(\hat{\pi}^t,\bar{x}^t,\bar{j}^t)\\
&\leq \tilde{O}\left(k^\frac{11}{4}s^\frac{3}{4}T^\frac{41}{45}\log\vert\cH\vert^\frac{1}{2}\right).
\end{align*}

And,
\begin{align*}
\sum_{t=1}^T Unfair^{\alpha,\gamma}(\hat{\pi}^t,\bar{x}^t,\bar{j}^t) &\leq \frac{1}{C\epsilon}\left[\tilde{O}\left(k^\frac{11}{4}s^\frac{3}{4}T^\frac{41}{45}\log\vert\cH\vert^\frac{1}{2}\right) - Regret^{err}(T)\right]\\
&\leq \frac{1}{C\epsilon}\left[\tilde{O}\left(k^\frac{11}{4}s^\frac{3}{4}T^\frac{41}{45}\log\vert\cH\vert^\frac{1}{2}\right) + kT\right]\\
&\leq\tilde{O}\left(\frac{1}{\epsilon}k^\frac{11}{4}s^\frac{3}{4}T^\frac{41}{45}\log\vert\cH\vert^\frac{1}{2}\right).
\end{align*}
\end{proof}

%% file: conclusion.tex
\section{Conclusion and Future Directions}\label{sec:conclusion}

Our work suggests a number of future directions. First, the Exp2 algorithm has runtime and space requirements that are linear in $\vert\cH\vert$, which is prohibitive for large hypothesis classes. Context-Semi-Bandit-FTPL is oracle-efficient, but is limited only to classes for which small separator sets are known. We inherit these limitations from the contextual bandit literature --- they hold even without the additionally encoded fairness constraints. Second, our adaptation of Context-Semi-Bandit-FTPL 
requires $T$ additional oracle calls at each iteration, to estimate the implicit distribution by the learner. Taken together, these limitations suggest the following important open question: are there substantially simpler and more efficient algorithms which can provide multi-criteria accuracy and fairness guarantees of the sort we give here using one-sided feedback with auditors? This question is interesting also in less adversarial settings than we consider here. For example, do things become easier if the panel is selected i.i.d. from a distribution every round, rather than being chosen by an adversary?

%% file: appendix.tex











\section{Omitted Details from Section \ref{sec:algorithms}}\label{app:algorithms}






\begin{proof}[Proof of Lemma \ref{lem:gap}] 
We set $\alpha=0.2,\gamma=1$ and $k=2$. We define the context space to be $\cX = \{x,x'\}$, and the hypothesis class as $\cH = \{h,h'\}$, where $h(x) = h'(x') = 1$, and $h(x') = h'(x) = 0$. We set $m=1$, and the panel $\bar{j}^{\alpha,\gamma}$, that hence consists of a single auditor, to reflect the judgements of $j^\alpha$, where $d^j(x,x') = 0.1$. Finally, we define $\pi\in\Delta \cH$ to return $h$ with probability $0.5$, and $h'$ with probability $0.5$. We denote $\bar{x} = (x,x')$.

Next, note that
\begin{align*}
h(x)-h(x') = 1 &> 0.3 = d^j(x,x') + \alpha,\\     
h'(x')-h'(x) = 1 &> 0.3 = d^j(x,x') + \alpha.
\end{align*}

Hence,
\[
\E\limits_{h\sim\pi}\left[unfair^{\alpha,\gamma}(h,\bar{x},\bar{j})\right] = 0.5\cdot unfair^{\alpha,\gamma}(h,\bar{x},\bar{j}) + 0.5\cdot unfair^{\alpha,\gamma}(h',\bar{x},\bar{j}) = 1.
\]

On the other hand,
\[
\pi(x)-\pi(x') = \pi(x')-\pi(x) = 0 < 0.3 = d^j(x,x') + \alpha.
\]

Hence, 
\[
unfair^{\alpha,\gamma}(\pi,\bar{x},\bar{j}) = 0.
\]

Which proves the lemma.
\end{proof}